\documentclass[sigconf, nonacm]{acmart} 

\AtBeginDocument{%
  \providecommand\BibTeX{{%
    Bib\TeX}}}

\setcopyright{none}

\acmConference[KDD '26]{TODO}{TODO}{TODO}

\acmISBN{978-1-4503-XXXX-X/2018/06}

\usepackage{mathtools}
\usepackage{graphicx}
\usepackage{booktabs}
\usepackage{multirow}
\usepackage{hyperref}
\usepackage{amsmath, amsthm} 
\usepackage{bbm}
\usepackage{algorithm}
\usepackage{algpseudocode}
\usepackage{subcaption}
\usepackage[table]{xcolor}

\definecolor{verylightgray}{gray}{0.9}
\definecolor{myblue}{RGB}{0,90,180}

\newcommand{\vecX}{\mathbf{x}}

\newcommand{\R}{\mathbb{R}}
\newcommand{\vecDelta}{\boldsymbol \delta}

\begin{document}

\title{GENADA: efficient generative time series adversarial attack framework}

\author{Michael Baronov}
\authornote{These authors contributed equally to this research.}
\email{m.baronov@applied-ai.ru}
\author{Denis Vorobev}
\authornotemark[1]
\email{vorobiyov2005@gmail.com}
\author{Margarita Rusanova}
\authornotemark[2]
\email{mirusanova@gmail.com}
\affiliation{%
  \institution{$^*$Moscow Independent Research Institute of Artificial Intelligence}
  \institution{$^\dagger$HSE University}
  \city{Moscow}
  \country{Russia}}




\author{Petr Sokerin}
\affiliation{%
 \institution{Moscow Independent Research Institute of Artificial Intelligence}
  \city{Moscow}
  \country{Russia}}
\email{sokerinpo@mail.ru}



\author{Alexey Zaytsev}
\affiliation{%
  \institution{Moscow Independent Research Institute of Artificial Intelligence}
  \city{Moscow}
  \country{Russia}}
\email{likzet@gmail.com}


\begin{abstract}
Deep learning models are widely used for time series analysis in domains such as healthcare, finance, energy systems, and environmental monitoring. 
However, these models remain vulnerable to adversarial attacks, where small input perturbations cause severe degradation in predictive performance. 
Commonly used gradient-based attacks, iterative first-order methods, are computationally burdensome, as they repeatedly backpropagate through the victim model to compute input gradients during a number of iterative refinement steps.
We propose a GENerative ADversarial Attack (GENADA) that learns a generative model to produce deceptive perturbations directly in a single forward pass and a procedure to train it.
Variants include single-step and iterative generative attack schemes.
The validation considers attacks on several neural models and datasets in the time-series domain, a controlled, low-dimensional setting. 
Empirically, GENADA achieves comparable attack quality to strong baselines while requiring less time to generate perturbations during inference.

The code of our project is available in the anonymized github-repository: \href{https://anonymous.4open.science/r/adversarial-gen-attacks-6690}{https://anonymous.4open.science/r/adversarial-gen-attacks-6690}.


\end{abstract}




\maketitle

\begin{acks}
The work was supported by the grant for research centers in the field of AI provided by the Ministry of Economic Development of the Russian Federation in accordance with the agreement 000000C313925P4F0002 and the agreement №139-10-2025-033.
\end{acks}

\section{Introduction}\label{sec:intro}
Sequential data arise in a wide range of applications, including healthcare, finance, energy systems, and environmental monitoring~\citep{marusov2023noncontrastive}. One of the central problems in this domain is time series classification, which plays a crucial role in forecasting, anomaly detection, and medical diagnosis~\citep{ismail2019deep}. Neural networks are among the most effective model families for time series classification tasks~\citep{ismail2019deep}.

Despite their strong predictive performance, neural networks are vulnerable to adversarial attacks: small, often imperceptible perturbations of the input may lead to significant drops in classification accuracy. While adversarial attacks have been extensively studied in computer vision~\citep{akhtar2018threat} and natural language processing~\citep{goyal2023survey}, time series data remain comparatively underexplored. At the same time, sequential signals possess specific structural properties—such as temporal correlations, smoothness constraints, and physical plausibility—that must be taken into account when designing attacks~\citep{sokerin2025concealed}.

Adversarial attacks are commonly categorized into white-box attacks, where gradients of the target model are available, and black-box attacks~\citep{zhang2022investigating}, where such access is absent. Although white-box attacks typically achieve higher attack effectiveness, they require access to the internal model, which is often unrealistic in practice. Moreover, because differentiation of the model tends to be slow, gradient-based attacks can be computationally expensive.

In this work, we propose a generative adversarial attack framework in which a separate neural network is trained to generate perturbations that maximize the target classifier's loss. Importantly, once trained, the generator can produce adversarial perturbations without computing gradients of the target model at inference time.

Our contributions are as follows:
\begin{enumerate}
\item We propose GENADA (GENerative ADversarial Attack), a generative adversarial attack that learns to produce valid perturbations directly. Once trained, the attacker outputs an adversarial perturbation 
by forward pass, avoiding per-example iterative optimization and eliminating the need for repeated gradient computations at inference time. The proposed attack is straightforward to train with standard learning pipelines and is substantially more efficient to run.

\item We propose a training procedure based on a frozen target model, enabling gradient-free inference. We train the generator to create perturbations that worsen the generator's predictions the most. We also provide the possibility of fine-tuning the training parameters, such as changing the number of trained iterations, the possibility of adding an additional step ~---~ distillation of iFGSM to our generator.

\item We validate the approach in the time-series domain, which provides a controlled low-dimensional setting for systematic experimentation while still capturing structured, temporally dependent signals. Empirically, we observe that the GENADA attains superior computational efficiency to strong baselines (in terms of the task-relevant degradation metrics) with comparable attack quality. The experiments are conducted on PowerCons, GunPoint and Strawberry datasets with metrics achieved , compared to metrics for the strongest alternative. We provide an extensive empirical evaluation across multiple datasets and architectures (recurrent, convolutional, and transformer-based models), comparing against strong white-box baselines such as FGSM and iFGSM.

\item We demonstrate that the proposed approach achieves competitive performance while significantly reducing attack generation time.

\end{enumerate}

\section{Related works}

\subsection{Adversarial Attacks}

Deep neural networks are known to be susceptible to adversarial perturbations~\citep{Szegedy2014}. Adversarial example generation methods are commonly categorized by attacker access (white-box vs. black-box), distortion measure (e.g., $\ell_p$-norms), and attack principle (gradient-based, optimization-based, or generative-based) \cite{akhtar2018threat}. In general, adversarial example generation aims to maximize the target model loss within an $\epsilon$-ball around the original input.

Classical white-box methods rely on gradients of the classification loss $\mathcal{L}_{\mathrm{cls}}$ with respect to the input $\mathbf{X}$. FGSM, iFGSM, and PGD remain standard first-order attack baselines.

\subsection{Time series classification}

Time series classification (TSC) is a central task in time series analysis~\cite{yang2006challenging} and appears in many practical settings, including finance and healthcare.

Standardized benchmark suites such as the UCR archive \cite{Dau2019_UCR} and the multivariate UEA archive \cite{bagnall2018ueamultivariatetimeseries} have enabled reliable comparisons between TSC methods. Following this benchmarking landscape, we selected three diverse datasets for univariate TSC.

Modern TSC employs a wide range of architectures, with no universal winner across datasets. Convolutional models such as ResCNN \cite{xiaowu2019rescnn} use residual connections \cite{he2016deep} to capture multi-scale patterns. Recurrent architectures including LSTMs \cite{10.1162/neco.1997.9.8.1735} and attention-based variants such as RNNAttention \cite{tsai} model temporal dependencies. More recently, Transformer-based models such as PatchTST \cite{nie2022time} and state-space approaches such as S4 \cite{gu2021efficiently} have shown strong performance on long sequences.

Before the rise of deep learning \cite{mohammadi2024deep}, TSC was dominated by methods based on Dynamic Time Warping \cite{keogh2005exact} and feature-based ensembles such as COTE \cite{bagnall2015time} and Weasel \cite{schafer2017fast}. Hybrid ensembles remain competitive: HIVE-COTE 2.0 (HC2) \cite{middlehurst2021hive,middlehurst2024correction} combines multiple complementary components and remains a strong non-deep baseline.






\subsection{Adversarial Attacks for Sequential Data}

While adversarial attacks are well studied in computer vision, they remain comparatively less explored in time series classification (TSC) despite the growing adoption of deep learning models \cite{Ding_Zhang_Feng_Huang_Jiang_Yang_2023}.

A common starting point is gradient-based white-box attacks, which expose model vulnerabilities across domains \cite{akhtar2018threat, tramer2017ensemble}. In the time-series context, Fawaz \cite{ifgsm} shows that straightforward iterative methods such as iFGSM and BIM can be effective, motivating their use as baseline attacks.

However, sequential data introduce constraints that are less prominent in images or text: temporal dependence, smoothness, and application-specific feasibility requirements. To address these issues, recent work \cite{SmoothPerturbations2022} adapts iterative attacks to time series and evaluates imperceptibility with criteria beyond simple norm bounds, including smoothness and bounded variation.

Several attacks have been designed specifically for time series. Adversarial transformation networks generate adversarial examples with a dedicated neural generator \cite{karim2020adversarial}, while subsequent work extends this approach to multivariate sequences \cite{harford2020adversarial}. Despite these advances, designing temporally consistent and transferable perturbations remains an active research direction.

\subsection{Surrogate Models and Transferability}

When gradients are unavailable, adversarial examples can be constructed using surrogate models. This approach relies on the transferability phenomenon \cite{Papernot2017} : models trained for the same task can exhibit similar gradient directions for increasing the loss function. Consequently, adversarial examples crafted for one model often transfer to other models trained to solve related tasks.

The effectiveness of such attacks depends on the alignment between the loss landscapes of surrogate and target models.
Accordingly, a surrogate model should be trained on the same training set and, ideally, use a similar architecture, hyperparameters, and other design choices.

Our approach differs in that we train a dedicated generative model to directly produce transferable perturbations for time series~data.

\section{Methodology}

\subsection{Problem statement}




\textbf{Time Series Classification.} 
Let $f(\vecX)$ denote a target classifier operating on 
time series $\vecX$. Here, the inference output $f(\vecX)$ is the predicted binary class label $y$ for the input time series $\vecX$. The dataset is given by 
$\mathcal{D} = \{(\vecX_i, y_i)\}_{i = 1}^n$,
where $\vecX_i \in \R^d$ is an input sequence and $y_i \in \mathcal Y$ is the corresponding class label.
The model $f$ is trained by minimizing cross-entropy loss~$
\mathcal{L}_{\mathrm{cls}}(f(\vecX), y)$. 

\textbf{Adversarial Attacks.} An adversarial attack constructs a perturbed example $\vecX'$ by adding perturbations $\boldsymbol{\delta}$ to the original data object
$$\vecX' = \vecX + \vecDelta,$$
where the perturbation $\vecDelta$ is constrained by $$
    \|\vecDelta\|_p \leq \epsilon
$$
and is optimized to induce incorrect or degraded predictions of the model $f$.
The problem is considered in a black-box scenario: we assume that the attacker has no access to the model $f$ or its weights and therefore cannot compute gradients with respect to the attacked model. The alternative white-box scenario, where the attacker has full access to the target model, is less realistic and is commonly used in classical gradient-based attacks such as FGSM. 

Another requirement is computational efficiency. The attack should be fast enough at inference time to be faster than, or at least comparable to, the classical attack methods considered in this work. This is necessary for the proposed attack to be practically applicable.

\subsection{Gradient-Based Baselines}
As a baseline, let's consider classical attacks based on explicit
gradients of the loss function of the target classifier \cite{ifgsm}.

\paragraph{Fast Gradient Sign Method (FGSM)}
Given a constraint $\|\boldsymbol{\delta}\|_\infty \leq \epsilon$, FGSM perturbs the input by taking a single step in the direction of the gradient sign:
\[
\vecX' = \vecX + \epsilon \, \mathrm{sign}\left(
\nabla_{\vecX} \mathcal{L}_{\mathrm{cls}}(\vecX, y)
\right).
\]
where $\vecX$ and $\vecX'$ denote the original and the attacked time series, respectively.

\paragraph{Iterative FGSM (iFGSM)}
A multi-step version refines the perturbation over $T$ iterations with step size $\epsilon /T\leq \alpha\leq \epsilon$, keeping the cumulative distortion within an $\ell_\infty$-ball of radius~$\epsilon$:

\[
\vecX^0 = \vecX, \quad
\vecX^{(t+1)} =
\operatorname{Clip}_{(\vecX, \epsilon)}
\left(
\vecX^{(t)} +
\alpha \,
\mathrm{sign}\left(
\nabla_{\vecX} \mathcal{L}_{\mathrm{cls}}(\vecX^{(t)}, y)
\right)
\right),
\]
where 
\[
\operatorname{Clip}_{(\vecX, \epsilon)} (\vecX') \coloneqq \min\{ \max\{\vecX' , \vecX - \epsilon\},  \vecX + \epsilon\}
\]

\textbf{Problems with gradient attacks}

Gradient-based attacks such as iFGSM or another strong attack PGD \cite{dong2022momentum} are inefficient because they perform computationally expensive backpropagation $T$ times over $T$ iterations.
This cost grows even more quickly when the model is large, when the attack must satisfy structured constraints, or when many candidate perturbations must be generated to achieve desired quality \cite{madry2018towards, goodfellow2015explaining, dong2022momentum}.

In contrast, a generative attack after training is expected to produce strong attacks in a single forward pass, making inference much faster and more scalable \cite{xie2020enabling, xiao2018generating}. 

\[
\vecX' = \vecX + a(G(\vecX))
\]
where $G(\cdot)$ - generative model for adversarial perturbations, $a(\cdot)$ - post-processing function.
This efficiency becomes especially important for iterative or multi-start attack strategies, where the attack procedure has to be repeated many times, for example, under different initial perturbations, random restarts, or candidate search procedures. In such settings, even a moderate per-iteration cost can make the overall attack prohibitively expensive. At the same time, a trained generator can amortize this cost and produce candidate perturbations much more efficiently.

Moreover, fine-tuning a generative model enables the incorporation of specific properties, such as learning perturbations that better respect domain structure, for example, temporal smoothness in time series \cite{su2026temporally, wang2024temporal}.

\subsection{GENADA}

\paragraph{Generative attacks (GENADA)}
Our goal is to develop a generative adversarial attack $h$ that manipulates the target model's predictions. To achieve this, our method employs an additional trainable attacking model $G: \mathbb{R}^d \rightarrow \mathbb{R}^d$ to generate adversarial perturbations. More formally, $\boldsymbol{\delta} = h(\vecX) = a(G(\vecX))$, where $a$ is a post-processing of the model outputs. In Section \ref{sec:gen_adv_attacks}, we provide a more detailed description of the generative attack and training procedures for generative models.

\subsubsection{GENADA}
\label{sec:gen_adv_attacks}

The main idea of a generative adversarial attack is to perturbations $\boldsymbol{\delta}$ not by using gradients of the target model, access to which may be limited, but by using a separate attacking model $G$. 
In this work, we propose two types of generative adversarial attacks: the generative adversarial attack $h_{G}$ and the iterative generative adversarial attack $h_{G_{\mathrm{iter}}}$.

In the $h_{G}$ attack, the attacking model takes the features as input and generates
an $\widetilde{\boldsymbol{\delta}}$ perturbation vector of the same dimension as the original data. However, such a vector cannot be directly used as a perturbation,
since its range is unlimited. To preserve the amplitude of the attack, we
use the regularization mechanism $a(\widetilde{\boldsymbol{\delta}})$: the output of the model $G(\vecX)$ is placed in
a non-linear function that limits the noise range. In our case, at the learning stage, we apply the differentiable hyperbolic tangent $\tanh(\cdot)$, which guarantees $\ell_\infty$-boundedness. Then multiply the noise by
the $\epsilon$ attack force, as in FGSM.

We use a generator network $G$ that produces perturbations~\cite{Baluja2018}:
\[
\boldsymbol{\delta} = \epsilon \, \tanh(G(\vecX)).
\]


\begin{figure*}[!h]
  \centering
  \includegraphics[width=0.7\textwidth]{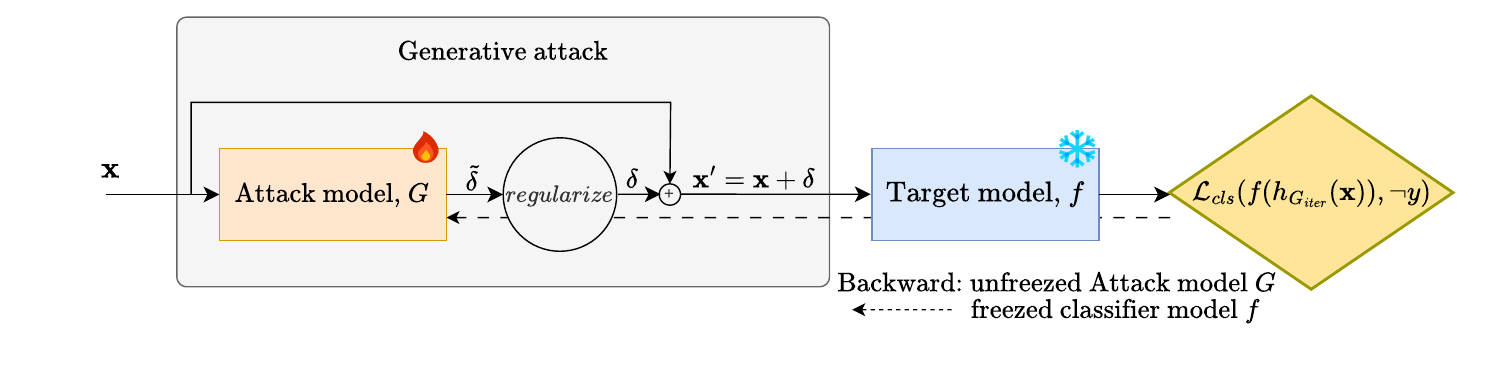}
  \caption{GENADA pipeline}
  \label{fig:gen_attack}
\end{figure*}

The generator is trained to minimize the classification loss with the inverted binary label of a frozen target model for its parameters $\boldsymbol{\theta}_G$:
\[
\mathcal{L}_{\mathrm{cls}}\big(f(\vecX + \epsilon \tanh (G_{\boldsymbol{\theta}_G}(\vecX)), \neg y\big) \rightarrow_{\boldsymbol{\theta}_G} \min.
\]

The backpropagation algorithm and gradient \textit{descent} with respect to the parameters of the generator model $G$ were used.

A detailed attack scheme is shown in Figure~\ref{fig:gen_attack}.

During inference, when evaluating attack metrics with a trained generator $G$, the regularization function $a$ was replaced with $\operatorname{sign}(\cdot)$. This ensures that the $\ell_\infty$ norm of the attack $\boldsymbol{\delta}$ is scaled exactly to $\epsilon$, whereas $\tanh(\cdot)$ merely constrains the attack within an $\ell_\infty$ bound. During training, we opted for $\tanh(\cdot)$ because, unlike $\operatorname{sign}(\cdot)$, it is smooth, which allows gradients to flow during training of the generator.

\subsection{Iterative GENADA training}

Analogous to iFGSM, we propose an iterative version in which perturbations are applied over $T$ steps. The attack noise generation step at each iteration corresponds to $\boldsymbol{\delta} / T$ so that the resulting attack conforms to $ \boldsymbol{\delta}$, while maintaining the total $\ell_\infty$ constraint.

Here, similarly, during training we used $\tanh(\cdot)$ as the regularization function $a(\cdot)$ to improve gradient flow, while at inference with the trained generator we applied $sign(\cdot)$ to increase the attack strength in terms of the $\ell_\infty$ norm.

To reduce memory consumption during training, we avoid full backpropagation in all $T$ iterations. Instead, gradients are computed only for the last $k$ steps, whereas the first $T-k$ steps are kept frozen. In addition, we propose a new training strategy in which the model is first trained in the $k$-step regime, and the number of trainable steps is gradually increased across epochs until it reaches $T$. This allows the generator to first learn a low-iteration attack and then extend it to the full multi-step setting, thereby accelerating training and reducing memory overhead. An example of a training scheme is shown in Figure ~\ref{fig:iter_attack_staregy}.

\begin{figure}[!h]
  \centering
  \includegraphics[width=0.9\columnwidth]{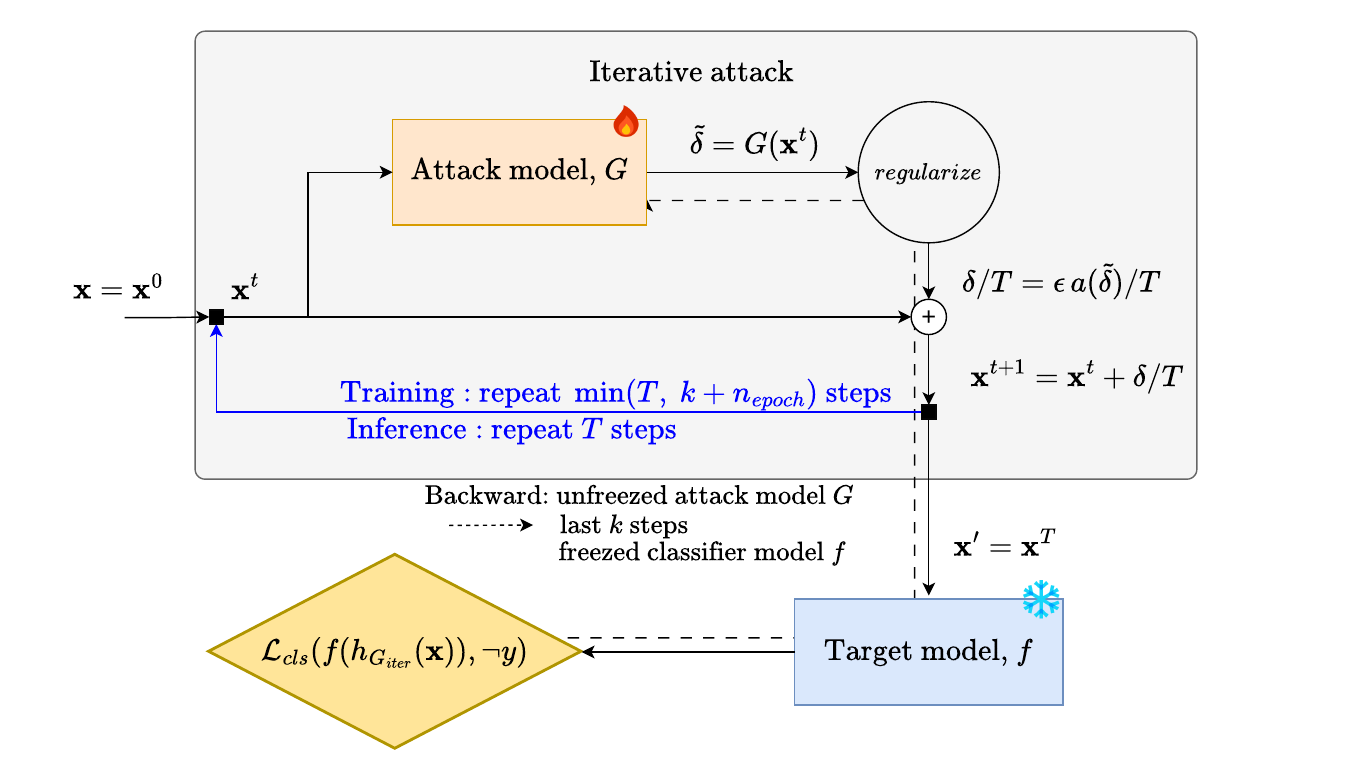}
  \caption{Iterative GENADA pipeline}
  \label{fig:iter_attack}
\end{figure}

The principle of an iterative attack is to attack data sequentially $T$ times, see Figure~\ref{fig:iter_attack} and Algorithm~\ref{alg:iter_genada}.
Similar to an iterative version of FGSM, we balance the performance drop after the attack and the computational cost by selecting~$T$.

\begin{figure}[!h]
  \centering
  \includegraphics[width=0.3\textwidth]{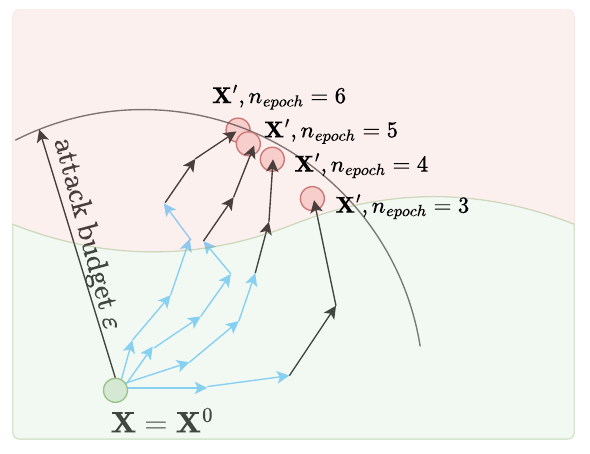}
  \caption{Example of an iterative attack during training with $T = 6$ and $k = 2$ unfrozen iterations at the end of the attack. In the first epoch, the $k = 2$ iterative attack is used; with each subsequent epoch, the iterativity increases by $1$. The \textcolor{myblue}{blue} steps are frozen, the gradient does not flow through them. Here, the classification classes 0 and 1 are indicated in red and green, respectively.}
  \label{fig:iter_attack_staregy}
\end{figure}

\begin{algorithm}[t]
\caption{Iterative GENADA inference}
\label{alg:iter_genada}
\begin{algorithmic}[1]
\Require Input time series $\vecX$, label $y$, frozen classifier $f$, generator $G$, attack budget $\epsilon$, number of steps $T$ \Ensure $\vecX^{(T)}$ \Comment{adversarial example}

\State $\vecX^{(0)} \gets \vecX$ \Comment{initialize the iterative process from the original input}

\For{$t = 0, 1, \dots, T-1$} \Comment{perform $T$ refinement steps}
    \State $\widetilde{\boldsymbol{\delta}}^{(t)} \gets G(\vecX^{(t)})$ \Comment{the generator predicts a perturbation}
    
    \State $\boldsymbol{\delta}^{(t)} \gets \frac{\epsilon}{T}  a \left(\widetilde{\boldsymbol{\delta}}^{(t)}\right)$ \Comment{scale the perturbation to control the total strength}
    
    \State $\vecX^{(t+1)} \gets \vecX^{(t)} + \boldsymbol{\delta}^{(t)}$ \Comment{update the adversarial example by adding the generated perturbation}
    
    \State $\vecX^{(t+1)} \gets \operatorname{Clip}_{(\vecX,\epsilon)}\!\left(\vecX^{(t+1)}\right)$ \Comment{project $\vecX^{(t+1)}$ into the allowed $\ell_\infty$-ball around $\vecX$}
\EndFor

\State \Return $\vecX^{(T)}$ \Comment{output the final adversarial example}
\end{algorithmic}
\end{algorithm}

\subsection{Theoretical Problem Setup}

Let \(x \in \mathbb{R}^d\) be an input and \(y \in \mathcal Y\) its label.
Let \(C:\mathbb{R}^d \to \mathbb{R}^K\) be a fixed classifier and
\(\ell:\mathbb{R}^K \times \mathcal{Y} \to \mathbb{R}\) a loss function.

\[
f_x(\delta) := \ell(C(x+\delta), y),
\qquad \|\delta\|\le \varepsilon .
\]

We consider the maximization problem
\[
    f_x(\delta) \rightarrow \max 
    \quad \text{s.t. } \|\delta\| \leq \epsilon.
\]

\[
\delta_x^\star = \arg\max_{\|\delta\|\le \varepsilon} f_x(\delta).
\]

Let the generator be a parametric mapping
\[
G_\theta : \mathbb{R}^d \to \mathbb{R}^d,
\]
and we train it by \textbf{maximizing} the objective
\[
J(\theta) := \mathbb{E}_{x \sim \mathcal{D}} \, f_x(G_\theta(x)).
\]

This setup formalizes GENADA's motivation: instead of solving a separate gradient-based maximization problem for each input $x$, we learn a generator $G_\theta(x)$ that directly produces perturbations. Under the assumptions stated in Appendix~\ref{sec:theor_justification}, the theorem shows that the quality gap is controlled by the generator approximation error and the remaining optimization error.

The theoretical justification and proof are presented in the Appendix \ref{sec:theor_justification}.

\subsection{Distillation training}
\label{sec:distilled_genada}

\paragraph{Motivation.}
The GENADA generator from Section~\ref{sec:gen_adv_attacks} requires white-box access to the target classifier $f$: every training step back-propagates $\mathcal{L}_{\mathrm{cls}}$ through $f$. In many practical settings the gradients of $f$ are unavailable, while a limited offline budget of expensive gradient-based queries (e.g., iFGSM) can still be afforded. We exploit this asymmetry by \emph{distilling} an iterative teacher attack into a single-forward generative student. After a one-off offline rollout of the teacher, the student is trained without any further access to $\nabla f$ and at inference produces an attack in one forward pass through a small generator. We refer to this procedure as a \emph{distilled attack}.

\paragraph{Teacher.}
The teacher attack $h^{\mathrm{T}}$ is iFGSM with $T^{\mathrm{T}}$ steps and an $\ell_\infty$-budget $\epsilon^{\mathrm{T}}$:
\[
\vecX^{(0)} = \vecX, \qquad
\vecX^{(t+1)} =
\operatorname{Clip}_{(\vecX,\epsilon^{\mathrm{T}})}\!\Big(
\vecX^{(t)} + \tfrac{\epsilon^{\mathrm{T}}}{T^{\mathrm{T}}}\,
\mathrm{sign}\!\big(\nabla_{\vecX}\mathcal{L}_{\mathrm{cls}}(f(\vecX^{(t)}), y)\big)
\Big).
\]
Given a training set $\mathcal{D} = \{(\vecX_i, y_i)\}_{i=1}^n$, we pre-compute and cache the teacher perturbations once:
\[
\boldsymbol{\delta}^{\mathrm{T}}_i \;\coloneqq\; \vecX^{(T^{\mathrm{T}})}_i - \vecX_i,
\qquad i = 1,\dots,n.
\]
All subsequent updates of the student rely on this cache and never query $\nabla f$ again. Note that with $T^{\mathrm{T}} > 1$ the entries of $\boldsymbol{\delta}^{\mathrm{T}}_i$ take multi-level values in $[-\epsilon^{\mathrm{T}}, +\epsilon^{\mathrm{T}}]$, not just the corners $\{\pm\epsilon^{\mathrm{T}}\}$ — a property that influences the choice of post-processing for the student.

\paragraph{Student.}
The student reuses the GENADA generator architecture $G_{\boldsymbol{\theta}_G}$ of Section~\ref{sec:gen_adv_attacks}, but applies the smooth post-processing $\tanh(\cdot)$ at \emph{both} training and inference:
\[
\boldsymbol{\delta}^{\mathrm{S}}(\vecX) \;=\; \epsilon^{\mathrm{S}}\,
\tanh\!\big(G_{\boldsymbol{\theta}_G}(\vecX)\big).
\]
For an iterative student with $T^{\mathrm{S}} \geq 1$ steps the construction is fully analogous to Algorithm~\ref{alg:iter_genada}:
\[
\vecX^{(0)} = \vecX, \qquad
\vecX^{(t+1)} = \vecX^{(t)} + \tfrac{\epsilon^{\mathrm{S}}}{T^{\mathrm{S}}}\,
\tanh\!\big(G_{\boldsymbol{\theta}_G}(\vecX^{(t)})\big).
\]
Because each step contributes at most $\epsilon^{\mathrm{S}}/T^{\mathrm{S}}$ in $\ell_\infty$ and $\tanh(\cdot)$ is bounded by~$1$, the cumulative perturbation satisfies $\|\boldsymbol{\delta}^{\mathrm{S}}\|_\infty \le \epsilon^{\mathrm{S}}$ without explicit clipping. Throughout this paper we use $T^{\mathrm{S}} = 1$ unless stated otherwise, so the distilled attack costs a single forward pass at inference.

The asymmetry with the white-box GENADA baseline (which uses $\mathrm{sign}(\cdot)$ at inference) is deliberate. The distillation target $\boldsymbol{\delta}^{\mathrm{T}}$ is produced by an \emph{iterative} attack and lives in $[-\epsilon^{\mathrm{T}}, +\epsilon^{\mathrm{T}}]^{d}$ with intermediate amplitudes, not only at the corners. A student quantized by $\mathrm{sign}(\cdot)$ at inference can only output the two extreme values and is therefore structurally unable to reproduce intermediate amplitudes of the teacher, which raises the irreducible distillation MSE. Keeping the smooth $\tanh(\cdot)$ at inference resolves this mismatch.

\paragraph{Distillation loss.}
The student is trained to mimic the cached teacher perturbations under a mean-squared error with an optional amplitude regularizer:
\[
\mathcal{L}_{\mathrm{dist}}(\boldsymbol{\theta}_G)
=
\underbrace{\mathbb{E}_{\vecX\sim\mathcal{D}}\,
\big\|\boldsymbol{\delta}^{\mathrm{S}}(\vecX) - \boldsymbol{\delta}^{\mathrm{T}}(\vecX)\big\|_2^2}_{\text{teacher matching}}
 + \alpha\,
\underbrace{\mathbb{E}_{\vecX\sim\mathcal{D}}\,
\big\|\boldsymbol{\delta}^{\mathrm{S}}(\vecX)\big\|_2^2}_{\text{amplitude penalty}},
\]
where $\alpha \ge 0$ controls how strongly the student is biased toward small perturbations, $\mathcal{D}$ is the original train dataset of non-attacked samples. The gradient $\nabla_{\boldsymbol{\theta}_G}\mathcal{L}_{\mathrm{dist}}$ flows through $G_{\boldsymbol{\theta}_G}$ and the smooth $\tanh(\cdot)$ only — the target $f$ is never queried at training~time.

\paragraph{Three $\ell_\infty$ budgets.}
We distinguish three budgets and treat them as separate knobs:
$\epsilon^{\mathrm{T}}$ — the teacher's $\ell_\infty$-budget that bounds the cached attack;
$\epsilon^{\mathrm{S}}_{\mathrm{train}}$ — the $\epsilon$ used inside $\boldsymbol{\delta}^{\mathrm{S}}$ when computing $\mathcal{L}_{\mathrm{dist}}$;
$\epsilon^{\mathrm{S}}_{\mathrm{eval}}$ — the $\epsilon$ applied to the trained student at inference.
We always set $\epsilon^{\mathrm{S}}_{\mathrm{eval}} = \epsilon^{\mathrm{T}}$ so that teacher and student are evaluated under the same constraint. The choice of $\epsilon^{\mathrm{S}}_{\mathrm{train}}$, by contrast, controls how strongly the post-processing $\tanh(\cdot)$ is driven toward saturation during training: a small $\epsilon^{\mathrm{S}}_{\mathrm{train}}$ pushes $G_{\boldsymbol{\theta}_G}(\vecX)$ to large magnitudes (saturated $\tanh$, sharp signs), while a large $\epsilon^{\mathrm{S}}_{\mathrm{train}}$ keeps $G_{\boldsymbol{\theta}_G}(\vecX)$ in the near-linear regime of $\tanh$ (soft, low-magnitude $\boldsymbol{\delta}^{\mathrm{S}}$).

\paragraph{Pipeline.}
The full procedure is summarized in Algorithm~\ref{alg:distilled_genada}. The teacher is queried only in line~1 to populate $\{\boldsymbol{\delta}^{\mathrm{T}}_i\}_{i=1}^n$; thereafter back-propagation runs only through $G_{\boldsymbol{\theta}_G}$ and the smooth $\tanh(\cdot)$, so the per-epoch training cost is $\mathcal{O}(n \cdot c_{G})$, where $c_{G}$ is the per-sample cost of one forward+backward pass through the generator — independent of the size and depth of $f$.

\begin{algorithm}[t]
\caption{Distilled attack}
\label{alg:distilled_genada}
\begin{algorithmic}[1]
\Require Dataset $\mathcal{D}=\{(\vecX_i, y_i)\}_{i=1}^n$; frozen classifier $f$; teacher attack $h^{\mathrm{T}}$ with budget $\epsilon^{\mathrm{T}}$ and $T^{\mathrm{T}}$ steps; generator $G_{\boldsymbol{\theta}_G}$; student budgets $\epsilon^{\mathrm{S}}_{\mathrm{train}}$, $\epsilon^{\mathrm{S}}_{\mathrm{eval}}$; student steps $T^{\mathrm{S}}$; regularization weight $\alpha$; number of distillation epochs $E$; learning rate $\eta$.
\Ensure Trained generator $G_{\boldsymbol{\theta}_G^{\star}}$.

\State \textbf{Teacher rollout:} compute and cache $\boldsymbol{\delta}^{\mathrm{T}}_i \gets h^{\mathrm{T}}(\vecX_i) - \vecX_i$ for all $i = 1,\dots,n$. \Comment{requires $\nabla f$; performed once}

\For{$e = 1,\dots,E$}
  \For{mini-batch $\mathcal{B} \subset \{1,\dots,n\}$}
    \State $\vecX^{(0)}_i \gets \vecX_i$ for $i \in \mathcal{B}$
    \For{$t = 0,\dots,T^{\mathrm{S}}-1$}
      \State $\vecX^{(t+1)}_i \gets \vecX^{(t)}_i + \tfrac{\epsilon^{\mathrm{S}}_{\mathrm{train}}}{T^{\mathrm{S}}}\,
              \tanh\!\big(G_{\boldsymbol{\theta}_G}(\vecX^{(t)}_i)\big)$
    \EndFor
    \State $\boldsymbol{\delta}^{\mathrm{S}}_i \gets \vecX^{(T^{\mathrm{S}})}_i - \vecX_i$
    \State $\mathcal{L} \gets \dfrac{1}{|\mathcal{B}|}\sum_{i \in \mathcal{B}}
            \Big( \big\|\boldsymbol{\delta}^{\mathrm{S}}_i - \boldsymbol{\delta}^{\mathrm{T}}_i\big\|_2^2
                  + \alpha\,\big\|\boldsymbol{\delta}^{\mathrm{S}}_i\big\|_2^2 \Big)$
    \State $\boldsymbol{\theta}_G \gets \boldsymbol{\theta}_G - \eta\,\nabla_{\boldsymbol{\theta}_G}\mathcal{L}$ \Comment{no gradient through $f$}
  \EndFor
\EndFor

\State \textbf{Inference:} for $\vecX$ set $\vecX^{(0)} = \vecX$ and iterate
\State \quad $\vecX^{(t+1)} = \vecX^{(t)} + \tfrac{\epsilon^{\mathrm{S}}_{\mathrm{eval}}}{T^{\mathrm{S}}}\,
       \tanh\!\big(G_{\boldsymbol{\theta}_G^{\star}}(\vecX^{(t)})\big)$ for $t = 0,\dots,T^{\mathrm{S}}-1$.
\State \Return $G_{\boldsymbol{\theta}_G^{\star}}$.
\end{algorithmic}
\end{algorithm}

\paragraph{Remarks.}
The distilled attack has two practical advantages over the white-box training baseline of Section~\ref{sec:gen_adv_attacks}:
\textbf{(i)} after the one-off teacher rollout, no further access to $\nabla f$ is required; we can complete the same step for a black-box attack to collect pairs of attack vectors and perturbed examples, so the procedure is compatible with a setup in which $f$ is exposed only through a fixed offline set of teacher perturbations;
\textbf{(ii)} the per-epoch training cost is independent of the size and depth of $f$, since back-propagation runs only through the much smaller generator $G_{\boldsymbol{\theta}_G}$.
The structural cost is a hard upper bound on attack strength: a perfectly distilled student inherits the fooling rate of the teacher and cannot exceed it. The empirical gap between teacher and student fooling rates — and the architectural choices that close or widen it — is the central object of study in Section~\ref{sec:experiments}.

\section{Experiments}
\label{sec:experiments}

\subsection{Datasets}

We evaluate on three binary classification datasets from the UCR archive~\citep{Dau2019_UCR}: PowerCons ($L=144$, 180/180 train/test), Strawberry ($L=235$, 613/370), and GunPoint ($L=150$, 50/150).
These datasets represent different application domains while sharing the same setting of univariate time-series classification.

\subsection{Model Architectures}

We consider three target classifier families: LSTM~\cite{10.1162/neco.1997.9.8.1735},
ResCNN~\cite{rescnn}, and PatchTST~\cite{nie2022time}.

Generators follow the corresponding backbone architecture and produce perturbations with the same dimensionality as the input series.
In addition, we evaluate RNNA~\cite{tsai} and S4~\cite{gu2021efficiently} as generator architectures.

\subsection{Evaluation Metrics}

In our experiments, we used several evaluation metrics, calculated over a test sample: fooling rate, effectiveness \cite{sokerin2025concealed}, and target accuracy. We also measure the inference time during attack generation.

\textbf{Fooling Rate (FR)} measures the proportion of instances that change their prediction after the attack:
\[
FR = \frac{1}{N} \sum_{i=1}^{N}
\mathbbm{1}\big(f(\mathbf{x}_i) \neq f(\mathbf{x}'_i)\big),
\]
where $N$ is the sample size, $\mathbf{x}_i$ is the original instance, $\mathbf{x}_i'$ is the adversarial instance, and $f(\cdot)$ is the target model.

\textbf{Attack Effectiveness (E)}: characterizes the degradation of the target model's performance. It is defined as the difference between $1$ and the $F_1$-score computed on the perturbed examples $\vecX'$:
\[
E = 1 - F_1(y^{\mathrm{true}}, f(\mathbf{x}')),
\]
where $y^{\mathrm{true}}$ denotes the ground truth class labels.

\textbf{Target Accuracy (TA)}: measures the proportion of instances that remain correctly classified after the attack. It is essentially the standard \textit{accuracy} metric computed on the perturbed examples:
\[
TA = \frac{1}{N} \sum_{i=1}^{N}
\mathbbm{1}(f(\mathbf{x}'_i) = y_i).
\]
where $y_i^{\mathrm{true}}$ is the ground truth label for the $i$-th instance.

\subsection{Experimental setup}

Experiments were run on a single NVIDIA RTX A5000 GPU.
Datasets were normalized prior to training.
We used Adam/AdamW with learning rates in $\{10^{-3},10^{-4}\}$ and batch size 64.
The perturbation budget was set to $\varepsilon\in[0.3,0.4]$.
Target models were trained for up to 100 epochs.
FGSM used a single step, while iFGSM used $T=10$ iterations.
GENADA/iGENADA were trained for up to 200 epochs with number of iterations $T\le 20$. The number of unfrozen attack steps at the end is set to $k=3$. 

\paragraph{Iterative training strategy}.
At epoch $0$, the initial number of training steps is also equal to $k=3$. Then, it increases by one at each epoch until it reaches $T$, and remains fixed afterward. In our experiments, we mainly use $T=15$ or $5$.


The source code is available in the anonymized GitHub repository: \href{https://anonymous.4open.science/r/adversarial-gen-attacks-6690}{https://anonymous.4open.science/r/adversarial-gen-attacks-6690}.

\subsection{Main results}

\begin{figure}
    \centering
    \includegraphics[width=0.8\linewidth]{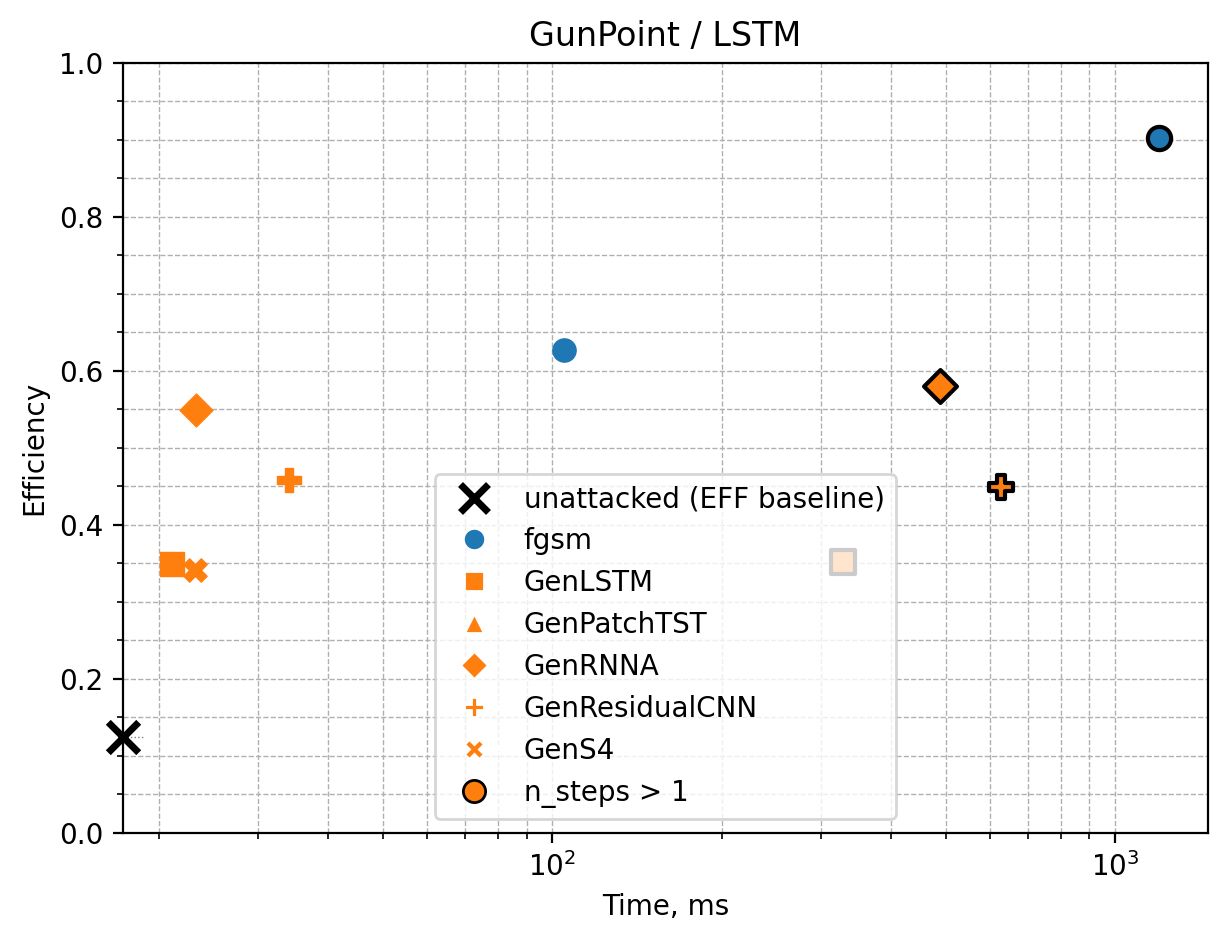}
    \caption{Time vs efficiency scatter plot}
    \label{fig:time_scatter}
\end{figure}

\begin{table*}[t]
\centering
\caption{Best fooling rate (FR) achieved by each attack type across all generators.}
\label{tab:results_fr_best}
\small
\setlength{\tabcolsep}{4pt}
\begin{tabular}{lccccccccc}
\toprule
\multirow{2}{*}{Attack} &
\multicolumn{3}{c}{PowerCons} &
\multicolumn{3}{c}{\cellcolor{verylightgray} GunPoint} &
\multicolumn{3}{c}{Strawberry} \\
& LSTM & PatchTST & ResCNN &
\cellcolor{verylightgray} LSTM &
\cellcolor{verylightgray} PatchTST &
\cellcolor{verylightgray} ResCNN &
LSTM & PatchTST & ResCNN \\
\midrule
FGSM
& \textit{0.361} & 0.382 & 0.382
& 0.533 & 0.876 & 0.064
& \textit{0.824} & 0.629 & 0.227 \\

iFGSM
& 0.113 & \textit{0.674} & \textbf{0.887}
& \textbf{0.820} & \textbf{0.969} & \textbf{0.989}
& \textbf{0.873} & 0.893 & \textit{0.969} \\

GENADA
& 0.185 & \textbf{0.661} & 0.456
& 0.713 & 0.773 & \textit{0.873}
& 0.789 & \textbf{0.951} & \textbf{0.973} \\

iGENADA
& \textbf{0.674} & 0.572 & \textit{0.744}
& \textit{0.773} & \textit{0.907} & 0.773
& 0.757 & \textit{0.9} & 0.53 \\
\bottomrule
\end{tabular}
\end{table*}













\paragraph{Trainability and generalization.}
All attacks reached a dataset- and architecture-specific performance plateau, typically within $10$--$100$ epochs. Transferability experiments show that the learned generators remain effective even when the generator and target classifier architectures differ, supporting their applicability in black-box settings.

\paragraph{Inference-time efficiency}
GENADA and iGENADA achieve the fastest inference while maintaining competitive attack quality (Figure~\ref{fig:time_scatter}). Once trained, the generator produces perturbations using only forward passes and does not require backpropagation through the target classifier.

\paragraph{Distilled attacks.}
Distillation performance is determined by two factors: teacher strength and the amount of training data available to learn the regression target $\boldsymbol{\delta}^{\mathrm{T}}(\mathbf{x})$. On Strawberry, the largest dataset, the best students nearly match the white-box teacher in fooling rate and achieve accuracy degradation comparable to iFGSM. On GunPoint, top students retain most of the teacher's fooling rate, while weaker generators degrade substantially due to the limited training set size. On PowerCons, student performance is constrained either by a weak teacher (LSTM target) or by the small dataset size, which limits distillation quality even when the teacher is stronger.

\paragraph{Generator hierarchy.}
A consistent ranking emerges across attack families. PatchTST is the most reliable backbone, achieving the strongest results under GENADA/iGENADA and remaining competitive under distillation. S4 dominates among distilled generators on larger datasets and is the only architecture that surpasses the iFGSM teacher on two Strawberry targets. RNNA occupies an intermediate position, while LSTM and ResidualCNN are consistently the weakest. Overall, the observed ordering is
$\textbf{PatchTST}, \textbf{S4} \succeq \textbf{RNNA} \succ \textbf{LSTM}, \textbf{ResidualCNN}$.

\paragraph{Artifacts and attack stealthiness.}
FGSM and iFGSM often produce visible high-frequency artifacts (Appendix~\ref{app:distilled_viz}, Figure~\ref{fig:examples}). GENADA can also introduce artifacts, although their magnitude is limited by the attack budget $\varepsilon$. In some cases, generative attacks are visually smoother than iFGSM, while in others the gradient-based attack remains more concealed. Improving attack stealthiness remains an important direction for future work, potentially through discriminator-based objectives~\citep{sokerin2025concealed}.

\section{Conclusion}

In this work, we introduced GENADA (GENerative ADversarial Attack), a generative framework for adversarial attacks on time series classification models. Unlike classical gradient-based approaches such as FGSM or iFGSM, the proposed method learns to generate perturbations directly using a dedicated neural generator. As a result, after the training stage, adversarial examples can be produced using only a forward pass through the generator, without iterative optimization and without repeated gradient computations with respect to the target model during inference. This significantly reduces the computational cost of attack generation and makes the method attractive for practical scenarios where low latency or restricted model access is important.

We proposed a training procedure based on a frozen target classifier, allowing the generator to learn perturbations that maximize the degradation of the classifier's predictions while preserving perturbation constraints. We evaluated the proposed framework on multiple time series classification datasets, including PowerCons, GunPoint, and Strawberry, and considered several model architectures, including recurrent, convolutional, and transformer-based networks. The experiments demonstrate that GENADA achieves attack quality competitive with strong white-box baselines while substantially reducing adversarial example generation time. In particular, the results show that generative attacks can successfully approximate the behavior of iterative optimization-based attacks while avoiding their high inference-time cost. The obtained results also indicate that generative approaches are capable of capturing these structures and producing effective perturbations even in temporally dependent signals.



Overall, the proposed framework demonstrates that learned generative attackers constitute a viable and efficient alternative to classical optimization-based adversarial attacks.

\paragraph{Limitations and future work}
Despite the obtained results, this study has several limitations. First, the proposed attacks are evaluated only in the binary classification setting, where the attack objective is to move the prediction toward the inverted class label. This setting is important as a first step, but multiclass classification is a more general and practically relevant scenario. 

Second, the quality and generalization ability of the learned attacks are still limited by the generator architecture and the training objective. In this work, we considered several neural architectures, including recurrent, convolutional, transformer-based, and state-space models. However, a more systematic analysis of new generator architectures, as well as ensemble-based attacks, may further improve transferability and attack effectiveness across different target classifiers. 





\bibliographystyle{ACM-Reference-Format}
\bibliography{bibliography}


\appendix

\section*{Appendix}

\begin{figure*}[!h]
  \centering

  \begin{subfigure}[t]{0.37\textwidth}
    \centering
  \end{subfigure}\hfill
  \begin{subfigure}[t]{0.37\textwidth}
    \centering
  \end{subfigure}

  \vspace{0.3em}

  \begin{subfigure}[t]{0.37\textwidth}
    \includegraphics[width=\linewidth]{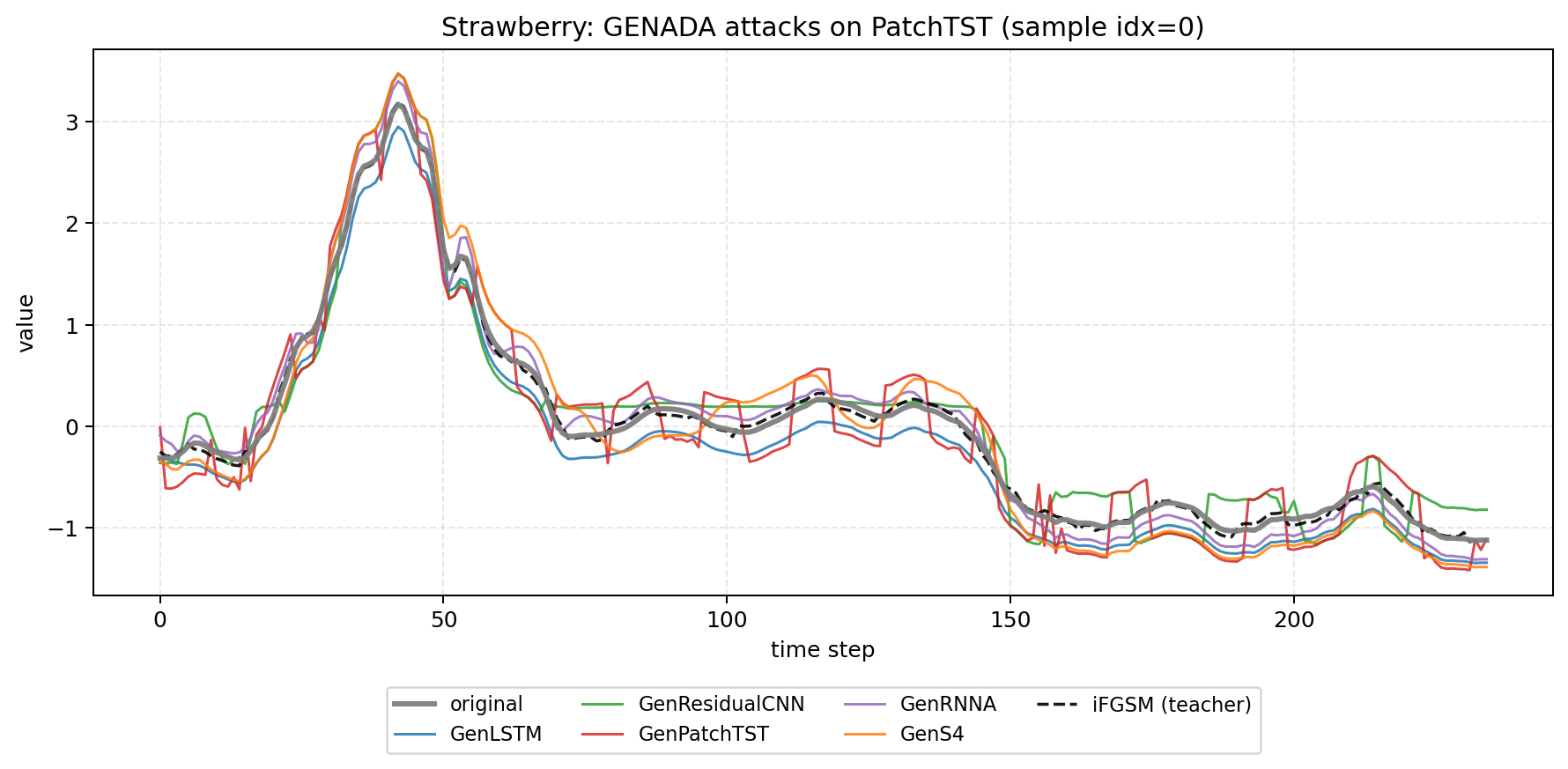}
    \caption{Strawberry — GENADA}
  \end{subfigure}\hfill
  \begin{subfigure}[t]{0.37\textwidth}
    \includegraphics[width=\linewidth]{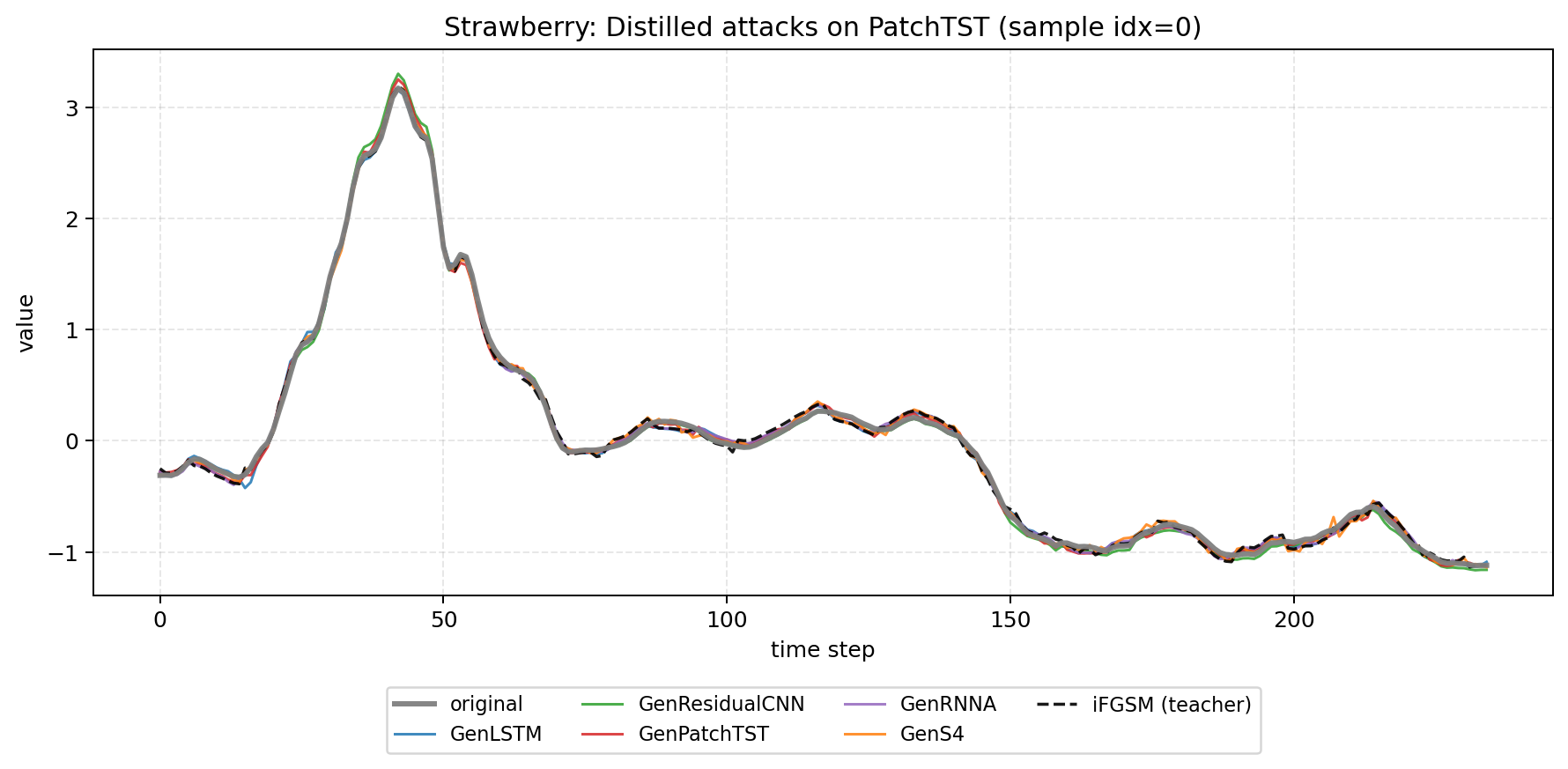}
    \caption{Strawberry — distilled}
  \end{subfigure}

  \vspace{0.5em}

  \begin{subfigure}[t]{0.37\textwidth}
    \includegraphics[width=\linewidth]{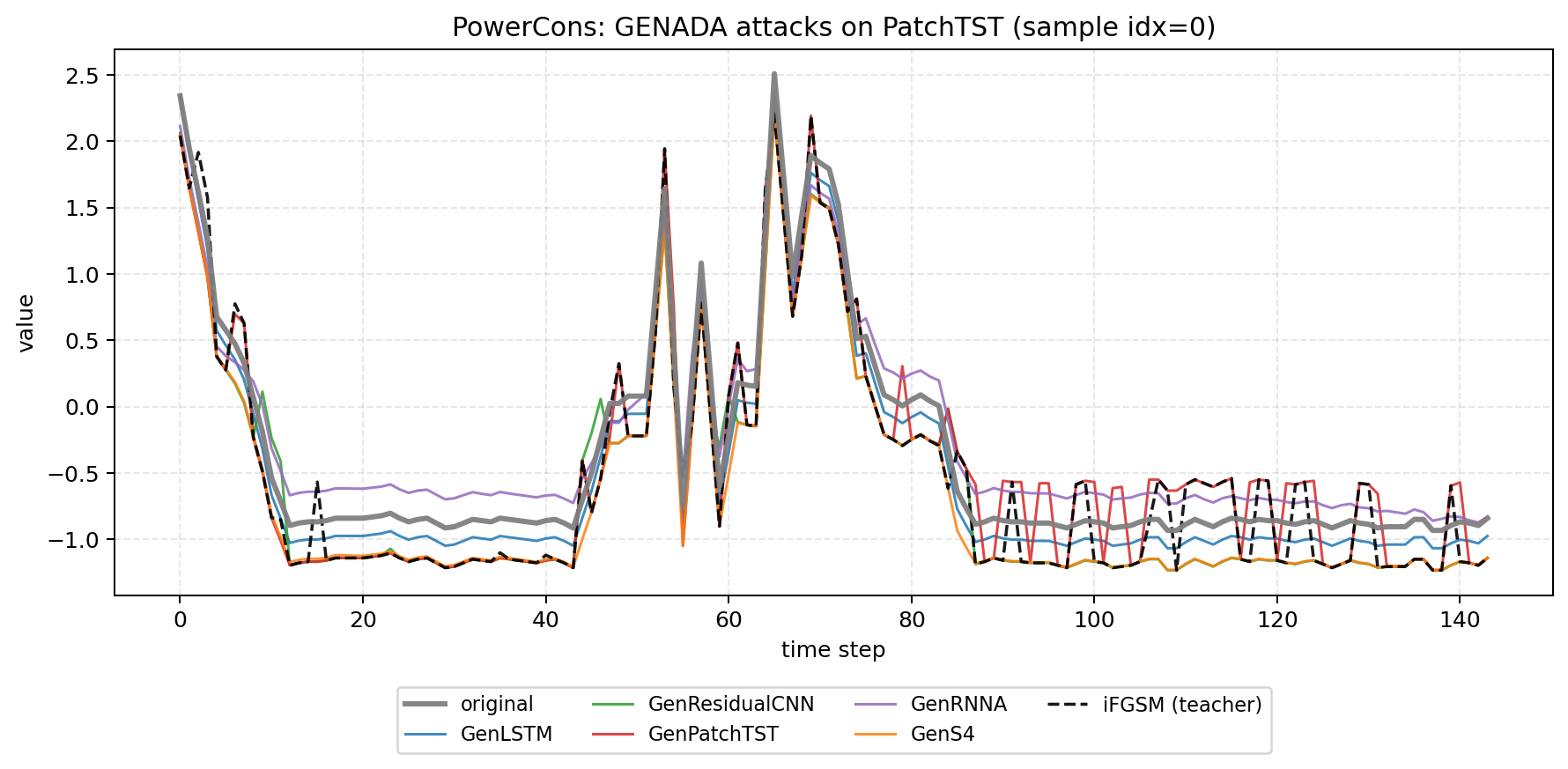}
    \caption{PowerCons — GENADA}
  \end{subfigure}\hfill
  \begin{subfigure}[t]{0.37\textwidth}
    \includegraphics[width=\linewidth]{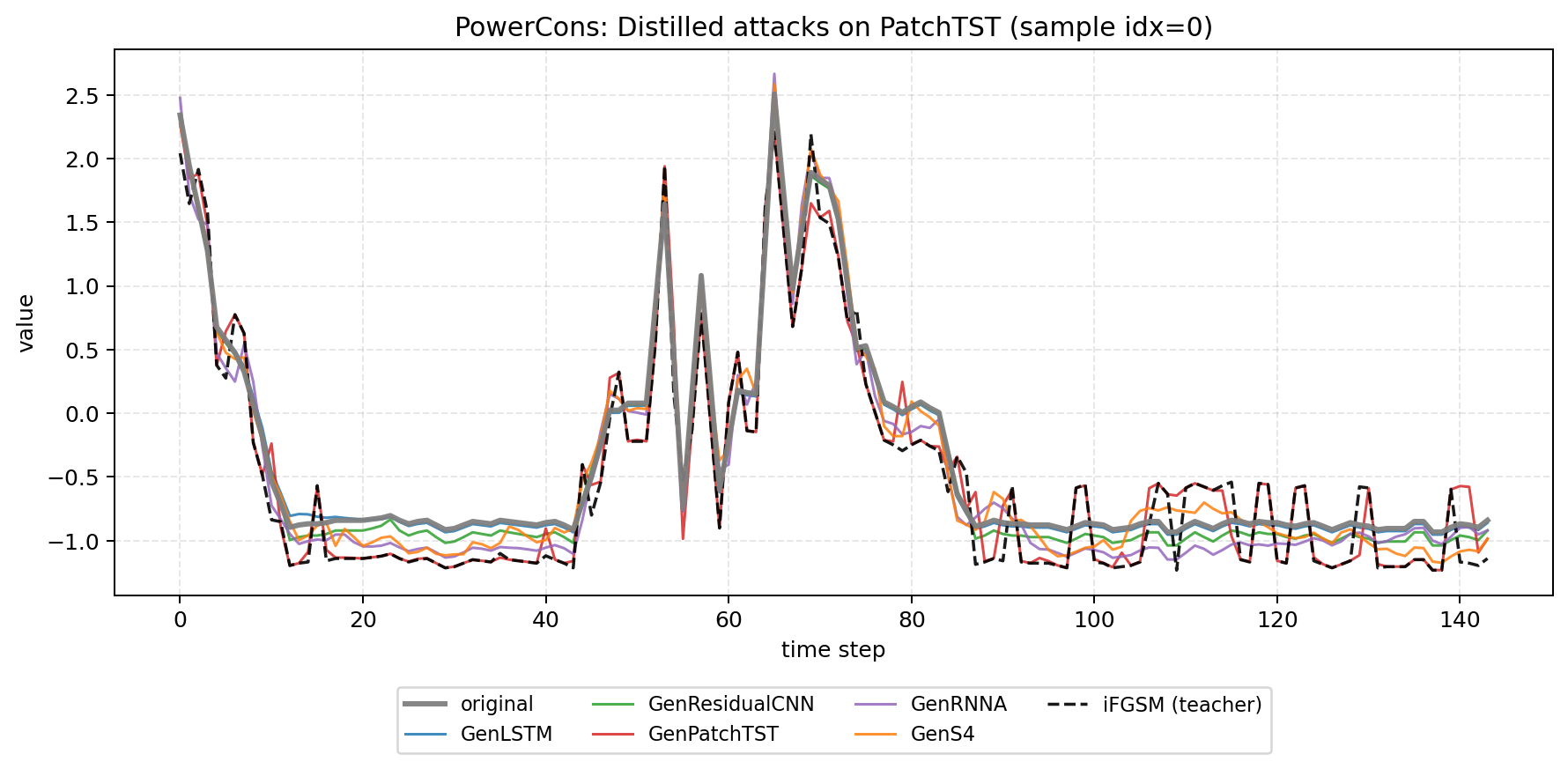}
    \caption{PowerCons — distilled}
  \end{subfigure}

  \vspace{0.5em}

  \begin{subfigure}[t]{0.37\textwidth}
    \includegraphics[width=\linewidth]{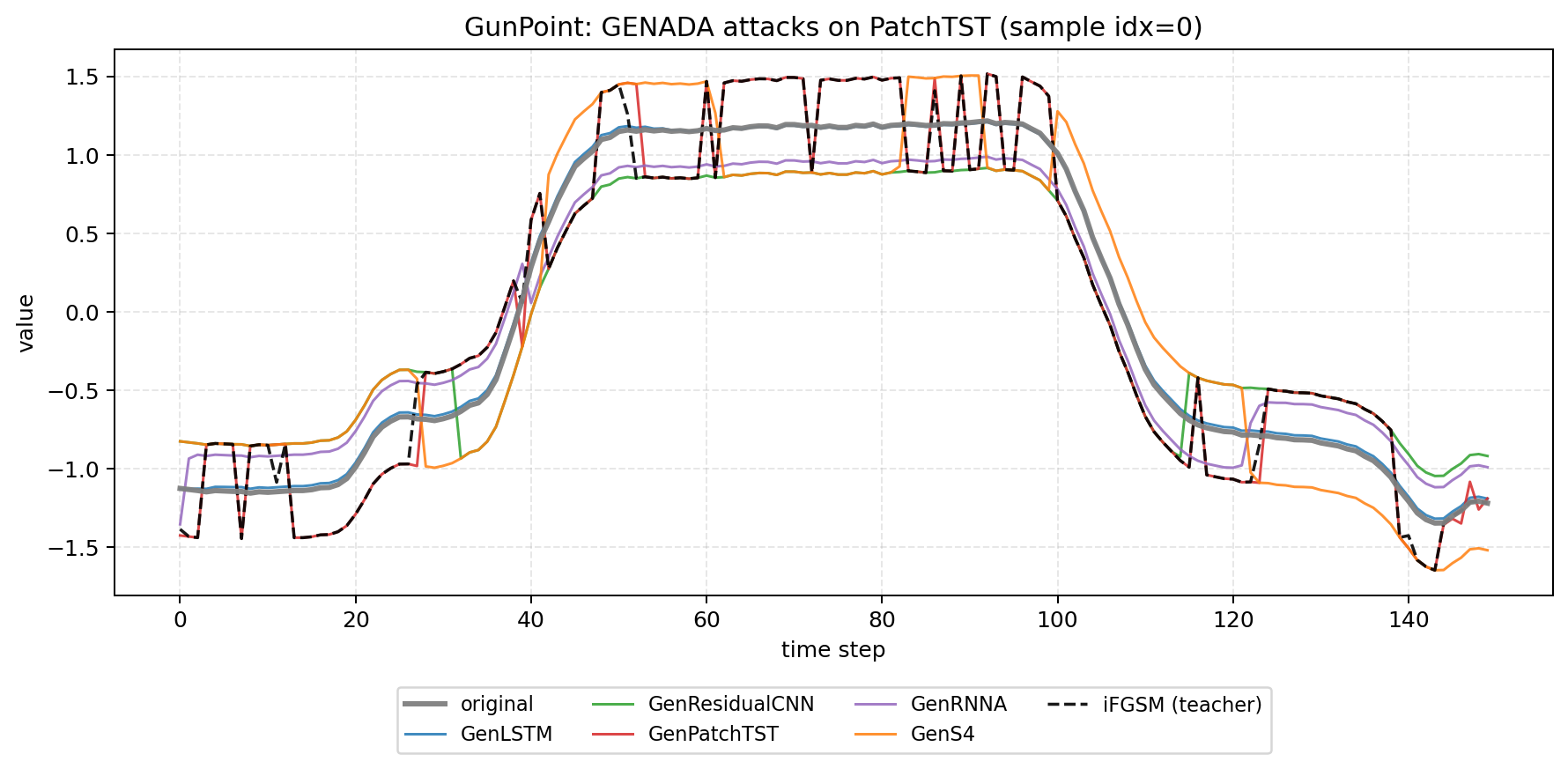}
    \caption{GunPoint — GENADA}
  \end{subfigure}\hfill
  \begin{subfigure}[t]{0.37\textwidth}
    \includegraphics[width=\linewidth]{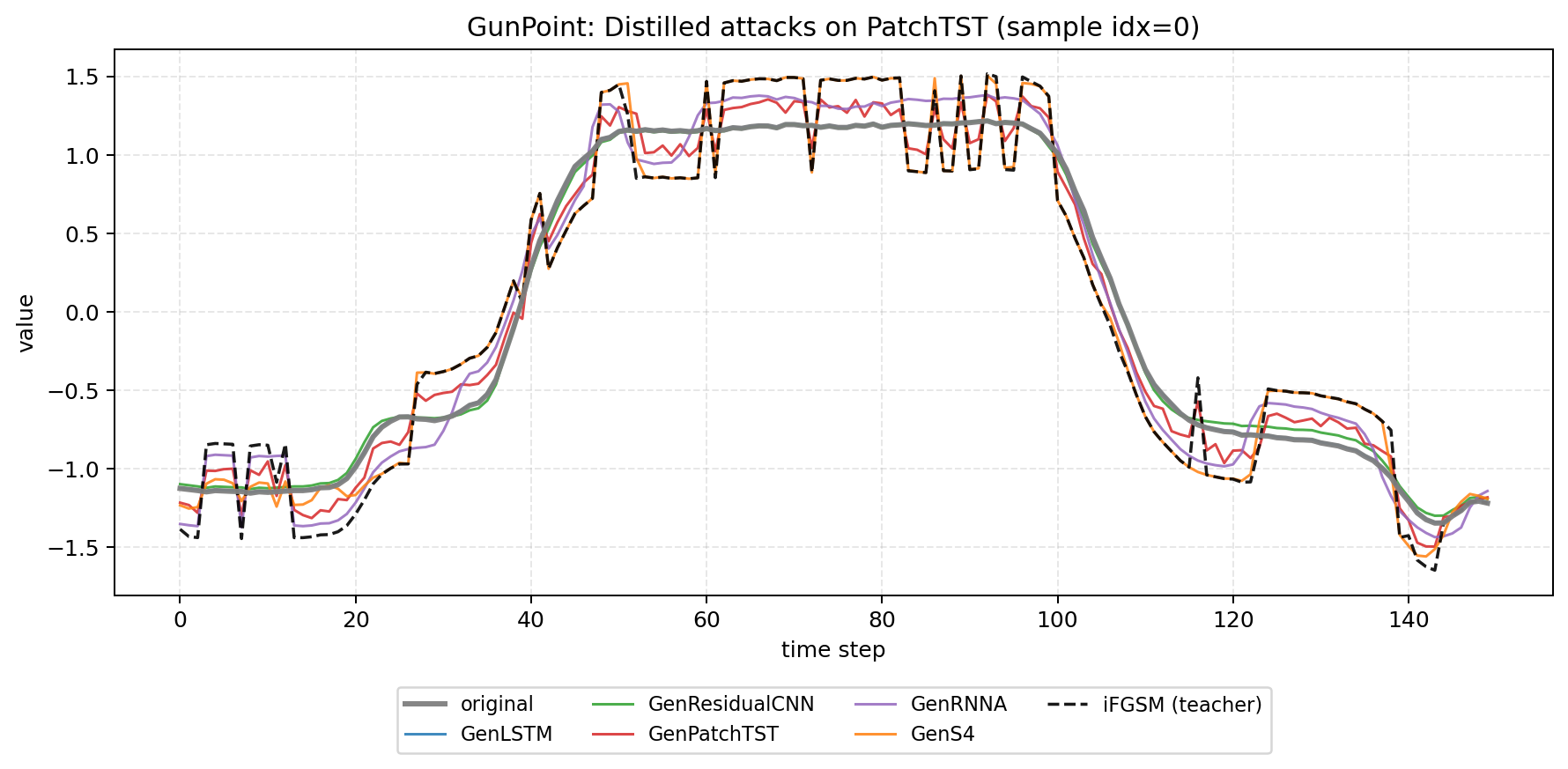}
    \caption{GunPoint — distilled}
  \end{subfigure}

  \caption{Adversarial perturbations on a single validation sample for PatchTST target.
  Left column: white-box GENADA; right column: distilled students.}
  \label{fig:examples}
\end{figure*}


\section{Theoretical justification of convergence}
\label{sec:theor_justification}
\subsection*{Theorem}

\subsubsection*{Assumptions}
The choice of these assumptions is motivated by a balance between mathematical tractability and the typical behavior observed in deep learning and adversarial robust optimization:

We consider three assumptions:
\begin{itemize}
\item[A1] $L$-smoothness of $f_x(\delta)$ with respect to \(\delta\):
\[
\forall x, \delta_1, \delta_2: \|\delta_i\|\le\varepsilon 
\Rightarrow 
\|\nabla f_x(\delta_1) - \nabla f_x(\delta_2)\|
\le
L \|\delta_1-\delta_2\|.
\]

This is a standard regularity assumption in non-convex optimization. In the context of adversarial attacks (like FGSM or PGD), assuming Lipschitz continuity of the loss gradient w.r.t. the perturbation $\delta$ is standard practice for proving convergence bounds. While ReLU activations can make the loss non-smooth at specific points, smoothed variants (like GELU or Softplus) or empirical observation of local smoothness under small $\varepsilon$ justify this assumption in practice.

\item[A2] Uniform approximation capacity of the generator:
\[
\forall \eta>0\; \exists \theta^\star: 
\sup_{x}
\|G_{\theta^\star}(x) - \delta_x^\star\|
\le \eta.
\]

This assumption relies on the Universal Approximation Theorems for neural networks. Since the perturbation space is bounded by $\|\delta\| \le \varepsilon$ and the input domain is typically compact (e.g., images in $[0, 1]^d$), a sufficiently wide or deep generator network $G_\theta$ is theoretically capable of uniformly approximating the optimal perturbation mapping $\delta_x^\star$. Similar expressivity assumptions are routinely used in the analysis of GANs and neural operators.

\item[A3] Polyak--Łojasiewicz (PL) condition for $J(\theta)$:
\[
\exists \mu_\theta>0: \forall \theta 
\Rightarrow 
\frac12 \|\nabla_\theta J(\theta)\|^2
\ge
\mu_\theta \big(J(\theta^\star) - J(\theta)\big).
\]

The PL condition is significantly weaker than strict concavity, as it allows for multiple local maxima and non-convex landscapes, which is exactly the case for neural network objectives $J(\theta)$. Recent deep learning theory literature shows that overparameterized neural networks often satisfy the PL condition locally or globally during gradient descent. Here, it elegantly captures the property that if the gradient of the generator's objective is small ($\varepsilon_{\mathrm{grad}}$), the generator is close to the optimal achievable utility $J(\theta^\star)$.

\end{itemize}

Now we can formulate theorem, that proofs workability of our method.

\begin{theorem}
Suppose the assumptions [A1]-[A3] hold, and let \(\theta_k\) satisfy
\[
\|\nabla_\theta J(\theta_k)\| \le \varepsilon_{\mathrm{grad}}.
\]
Then
\[
\mathbb{E}_{x}
\big[
f_x(\delta_x^\star)
-
f_x(G_{\theta_k}(x))
\big]
\le
\frac{L}{2}\eta^2
+
\frac{\varepsilon_{\mathrm{grad}}^2}{2\mu_\theta}.
\]
\end{theorem}

\begin{proof} \
\paragraph{I. Approximation error contribution.}

By the approximation capacity of the generator, for any \(\eta>0\) there exists \(\theta^\star\) such that
\[
\|G_{\theta^\star}(x) - \delta_x^\star\| \le \eta
\quad \forall x.
\]

By $L$-smoothness of $f_x$, it admits the quadratic upper bound
\[
\forall \delta_1, \delta_2 \Rightarrow 
f_x(\delta_1)
\le
f_x(\delta_2)
+
\langle \nabla f_x(\delta_2), \delta_1-\delta_2 \rangle
+
\frac{L}{2}\|\delta_1-\delta_2\|^2.
\]

Applying this inequality with
\(\delta_1 = G_{\theta^\star}(x)\) and
\(\delta_2 = \delta_x^\star\), we obtain
\[
f_x(G_{\theta^\star}(x))
\le
f_x(\delta_x^\star)
+
\langle \nabla f_x(\delta_x^\star),
G_{\theta^\star}(x) - \delta_x^\star \rangle
+
\frac{L}{2}\eta^2.
\]

The first-order optimality condition for $\delta_x^\star$ in the constrained maximization problem over the $\varepsilon$-ball yields
\[
\forall \delta': \|\delta'\| \leq \epsilon 
\Rightarrow 
\langle \nabla f_x(\delta_x^\star),
\delta' - \delta_x^\star \rangle
\le 0.
\]

Taking $\delta' = G_{\theta^\star}(x)$ gives
\[
\langle \nabla f_x(\delta_x^\star),
G_{\theta^\star}(x) - \delta_x^\star \rangle
\le 0,
\]
and therefore
$
f_x(\delta_x^\star)
-
f_x(G_{\theta^\star}(x))
\le
\frac{L}{2}\eta^2.
$

Taking expectation over \(x\):
$
J(\theta^\star)
\ge
\mathbb{E}_x f_x(\delta_x^\star)
-
\frac{L}{2}\eta^2.
$

\paragraph{II. Optimization error contribution.}

From the PL condition for $J(\theta)$ and \(\|\nabla_\theta J(\theta_k)\|\le\varepsilon_{\mathrm{grad}}\),
\[
\frac12 \varepsilon_{\mathrm{grad}}^2
\ge
\mu_\theta \big(J(\theta^\star) - J(\theta_k)\big),
\]
hence
$
J(\theta^\star) - J(\theta_k)
\le
\frac{\varepsilon_{\mathrm{grad}}^2}{2\mu_\theta}.
$

\paragraph{III. Final bound.}

\[
\mathbb{E}_x f_x(\delta_x^\star)
-
J(\theta_k)
=
\big(
\mathbb{E}_x f_x(\delta_x^\star)
-
J(\theta^\star)
\big)
+
\big(
J(\theta^\star) - J(\theta_k)
\big).
\]

Using the bounds from Steps I and II,
\[
\mathbb{E}_x f_x(\delta_x^\star)
-
J(\theta_k)
\le
\frac{L}{2}\eta^2
+
\frac{\varepsilon_{\mathrm{grad}}^2}{2\mu_\theta}.
\]

\end{proof}

\section{Visualization of distilled attacks}
\label{app:distilled_viz}

In this section, we present examples of samples attacked by different attacks, as shown in Figure~\ref{fig:examples}. 

\end{document}